\relax
\documentclass{article}
\usepackage{ijcai17}
\usepackage{times}
\usepackage{helvet}
\usepackage{url}
\usepackage{courier}
\usepackage{amsmath}
\usepackage{caption}
\usepackage{amsthm}
\usepackage{subcaption}
\usepackage{xcolor}
\usepackage{graphicx}
\usepackage{multirow}
\usepackage[ruled,vlined]{algorithm2e}
\usepackage{comment}
\usepackage{setspace} 
\usepackage{booktabs}
\newcommand{\us}{{\sc GFSE}}

\newtheorem{lemma}{Lemma}
\newenvironment{newlemma}[1]
  {\renewcommand\thelemma{#1}\lemma}
  {\endlemma}
\baselinestretch
\newcommand{\shorten}[1]{}

\newcommand{\bi}{\begin{list}{$\bullet$}{
    \setlength{\leftmargin}{1.5 em}
    \setlength{\itemsep}{0 pt}
    \setlength{\topsep}{3 pt}
    \setlength{\parsep}{3 pt}
    \setlength{\partopsep}{0 pt}
    \setlength{\labelwidth}{1 em}
    \setlength{\labelsep}{0.5 em}
    \setlength{\parskip}{0cm}  }}
\newcommand{\ei}{\end{list}}

\newcommand{\BE}{\begin{enumerate}}
\newcommand{\EE}{\end{enumerate}}

\newtheorem{theorem}{Theorem}
\newcommand{\initab}{                           % set up tab stops
\begin{tabbing}
XXX \= XXXX \= \kill
}
\newcommand{\begpub}{
\begin{quotation}
\noindent
}

\newcommand{\finpub}{
\end{quotation}
}

\begin{document}

\title{Sample Efficient Policy Search for Optimal Stopping Domains}
%\begin{comment}
\author{Karan Goel \\
%\thanks{Corresponding author (
%{\tt kgoel93@gmail.com} \\%)} \\
Carnegie Mellon University\\
{\tt kgoel93@gmail.com} 
\And Christoph Dann\\
Carnegie Mellon University\\
{\tt cdann@cdann.net}
% {\tt kgoel93@gmail.com}
% \And Christoph Dann\\
% Carnegie Mellon University\\
% 
% \And 
% Rika Antonova\\
% KTH\\
% {\tt rika.antonova@gmail.com}
\And 
Emma Brunskill\\
Stanford University\\
{\tt ebrun@cs.stanford.edu}
}
%\end{comment}
% \author{}
 \maketitle
\begin{abstract}
%Arising naturally in many fields, o
  Optimal stopping problems consider the question of deciding when to stop an observation-generating process in order to maximize a return. We examine the problem of simultaneously learning and planning in such domains, when data is collected directly from the environment. We propose \us, a simple and flexible model-free policy search method that reuses data for sample efficiency by leveraging problem structure. We bound the sample complexity of our approach to guarantee uniform convergence of policy value estimates, tightening existing PAC bounds to achieve logarithmic dependence on horizon length for our setting. We also examine the benefit of our method against prevalent model-based and model-free approaches on 3 domains taken from diverse fields.
 \end{abstract}

\section{Introduction}
Sequential decision making and learning in unknown environments, 
commonly modeled as reinforcement learning (RL), is a key aspect of 
artificial intelligence. An important subclass of RL is  optimal 
stopping processes, where an agent decides at each step whether to continue or terminate a stochastic process and the reward upon termination is a function of the observations seen so far. Many common problems in Computer Science and Operations Research can be modeled within this setting, including the secretary problem \cite{ferguson1989solved}, house selling \cite{glower1998selling,lippman1976economics}, American options trading \cite{jacka1991optimal,mordecki2002optimal}, product pricing \cite{feng1995optimal} and asset replacement \cite{DBLP:journals/ior/JiangP15}, as well as problems in 
artificial intelligence like mission monitoring robots \cite{best2015spatiotemporal}, metareasoning about the value of additional computation \cite{zilberstein1995operational} and 
automatically deciding when to purchase an airline ticket \cite{DBLP:conf/kdd/EtzioniTKY03}.  Often the stopping process dynamics are unknown in advance and so finding a good stopping policy (when to halt) requires learning from experience in the environment. As real experience can incur real losses, we desire algorithms that can quickly (with minimal samples) learn good policies that achieve high reward for these problems.

Interestingly, most prior work on optimal stopping has focused on the planning problem: how to compute near-optimal policies given access to the dynamics and reward of the stochastic stopping process \cite{peskir2006optimal}.  Optimal stopping problems can also be framed as a partially observable Markov decision process (POMDP), and there also exists work on learning a good policy for acting in POMDPs, that bounds the number of samples  required to identify the near optimal policy out of a class of policies~\cite{DBLP:conf/nips/KearnsMN99,DBLP:conf/uai/NgJ00}. However,  
such work either (i) makes the strong assumption that the algorithm has access 
to a generative model (ability to simulate from any state) of the stochastic process, which makes this 
work  more  suited to improving the efficiency of planning using simulations of the domain, or (ii) can use trajectories directly collected from the environment, but incurs exponential horizon dependence.

In this paper, we consider how to quickly learn a near-optimal policy in a stochastic optimal stopping process with unknown dynamics, given an input class of policies. % rephrase
We assume there is a fixed maximum length horizon for acting,  and  then 
make a  simple but powerful observation: for stopping problems with process-dependent rewards, the outcomes of a full length trajectory (that is,  a trajectory in which the policy only halts after the entire horizon) provide an estimated return for halting after one step, after two steps, and so on, till the horizon. In this way, a single full-length trajectory yields a sample return for any stopping policy. Based on this, we propose an algorithm that first  acts by  stopping only after the full length horizon for a number of trajectories, and then performs policy search  over an input policy class, where the full length trajectories are used to provide estimates of the  expected return of each policy considered in the policy class. The policy in the set with the highest expected performance  is selected for future use.  We provide sample complexity bounds on the number of full length trajectories sufficient to identify a near optimal policy within the input policy class.  Our results are similar to more general results for POMDPs~\cite{DBLP:conf/nips/KearnsMN99,DBLP:conf/uai/NgJ00}, but due to the structure of optimal stopping we achieve two key benefits: our bounds' dependence on the horizon is only logarithmic instead of linear (with a generative model) and exponential (without), and our results apply to learning in stochastic stopping processes, with no generative model required.  Simulation results on student tutoring, ticket purchase, and asset replacement show our approach significantly improves over  state-of-the-art approaches. %for reinforcement learning using policy search. 

\section{Problem Formulation}
%We will first lay out our problem definition formally, and then move on to describe our approach in the next section. %example from a tutoring simulation to give some intuition, as well establish the benefit of a model-free approach to these problems.

 We consider the standard stochastic discrete-time optimal stopping process setting. As in Tsitsiklis and Van Roy [\citeyear{tsitsiklis1999optimal}], we assume there is a stochastic process $\mathrm{P}$ that generates observations $o_1,\ldots,o_t$ (they may be vectors). There are two actions: halt or continue the process. The reward model is a known, deterministic function of the sequence of observations and the choice of whether to continue or halt. While there do exist domains where the reward model can be a nondeterministic function of the observations and the actions (such as a medical procedure that reveals the patient's true condition after a sequence of waiting), most common optimal stopping problems fall within the framework considered here, including the secretary problem (the quality of each secretary is directly observed), house selling (the price for the house from each bidder is known), asset replacement (published guides on the worth of an asset, plus knowledge of the cost of buying a new one), etc. We focus on the episodic setting where there is a fixed maximum time horizon for each process. 
The finite horizon value of a policy $V_\pi = \mathrm{E}_{\mathrm{P}}[ r |\pi]$ is the expected return from following $\pi$ over a horizon of $H$ steps, where the expectation is taken over the stochastic process dynamics $\mathrm{P}$. Note the policy may choose to halt before $H$ steps. The goal is to maximize return across episodes.

%As laid out by \cite{tsitsiklis1999optimal}, we consider the important subset of optimal stopping problems where the reward model is a known function of the observation history, and the current action choice but not of any other random events external to the process. We call these \emph{externally independent} optimal stopping problems. Many important problems are modeled within this setting, including sales environments with sale or purchase decisions \cite{DBLP:conf/kdd/EtzioniTKY03,zhang2001stock}, finance \cite{mordecki2002optimal} and operations research \cite{DBLP:journals/ior/JiangP15}. There exist other stopping problems where the reward model is externally \emph{dependent}; consider a stopping problem which involves monitoring a patient's condition before doing surgery -- the internal physiology of the patient is only indirectly observable until monitoring halts and surgery commences. \bug%\bug %Our theoretical results will also hold in both settings, but our method is most beneficial (in a practical sense) with externally independent rewards.

%The number of possible finite horizon histories (assuming a discrete observation space) is $|O|^H$, and so the number of possible policies is order $2^{|O|^H}$. 

% We assume that $\mathrm{P}$ is unknown to the agent but the return model is known. The goal is to maximize the return across many episodes.

We focus here on direct policy search methods (see e.g. \cite{DBLP:conf/nips/SuttonMSM99}). More precisely, we assume as input a parameterized policy class $\Pi = \{f(\theta)|\theta \in \Theta\}$ where $\Theta$ is the set of policy parameters. Direct policy search does not require building a model of the domain, and has been very successful in a variety of reinforcement learning (RL) contexts \cite{deisenroth2011pilco,levine2014learning}.

\section{Sample Efficient Policy Search}% Optimal Stopping Policy Search}
 We are particularly interested in domains where evaluation of a 
policy incurs real cost in the environment, 
such as stock market options selling. In such settings we 
wish to find sample efficient methods for doing policy 
search, that can minimize the number of poor outcomes 
in the real world. 
The challenge is that we do not know the stochastic dynamics 
$\mathrm{P}$ and so it is not possible to, in advance 
of acting, perform policy search to identify a good policy. 
Instead we can only obtain information about the domain 
dynamics by executing policies in the real world. We seek 
to efficiently leverage such experience to quickly make 
good decisions. 

We now present a simple approach, \us\ (Gather Full, Search and Execute) (Algorithm~\ref{alg:our_approach}), to do sample 
efficient policy search. \us\ collects a 
set of full-length (horizon $H$) 
trajectories, uses these to  
evaluate the performance of any policy in the input 
policy class $\Pi$, identifies a good policy, 
and then executes the resulting policy on all future 
episodes. 

\begin{algorithm}[t]
\SetAlgoLined
Input: policy class $\Pi$, search method $\mathcal{S}$, $\epsilon, \delta$\\
$n \leftarrow$ Use Theorem~\ref{theorem:pac} with $\epsilon,\delta$\\
$\Gamma \leftarrow$ Gather $n$ full trajectories from environment\\
% \KwResult{Write here the result }
%  initialization\;
$\pi* \leftarrow$ Identify policy using $\mathcal{S}(\Pi,\Gamma)$ \textbackslash\textbackslash   evaluation uses $\Gamma$\\
Execute $\pi*$
%  \While{not identified}{
%   instructions\;
%   \eIf{condition}{
%    instructions1\;
%    instructions2\;
%    }{
%    instructions3\;
%   }
%  }
 \caption{Gather Full, Search and Execute (GFSE)}
 \label{alg:our_approach}
\end{algorithm}

The key insight is in the first step, gathering the data 
to be used to evaluate the performance of any policy 
in the policy class $\Pi$. Monte Carlo 
estimation can be used to estimate the expected return of a policy 
by running it many times. However, this scales poorly 
with the cardinality of the policy class. Building a dynamics model 
from a set of data is more efficient, 
as a model can be used to simulate the performance of 
any policy, but this requires us to make certain assumptions about 
the domain (for ex. the Markov property) Which can lead to biased estimates. Alternatively, 
importance sampling can be used to do off-policy evaluation \cite{precup2000eligibility} but unfortunately such estimates tend to be very 
high variance.

However, a simple but powerful observation is that  a full-horizon ($H$-step) trajectory can be used to yield a sample return for all optimal stopping policies in $\Pi$.  Given a full length trajectory $(o_1,\ldots,o_H)$, the performance of a particular policy $\pi$ can be simulated by providing  $(o_1,o_2,\dots)$ to the target policy until it halts at some time step $t \leq H$. %Recall that we consider optimal stopping processes in which the return from an episode is a known determistic function of the history of prior observations\footnote{Often the return may simply be a function of the last observation or a subset of the last observation if observations are vectors, such as when buying airline tickets.}. 
Therefore we can take the subsequence of observations $(o_1,o_2,\ldots,o_{t})$ and use it to directly compute the return that would have been observed for executing $\pi$ on this trajectory.   A single full-horizon trajectory will provide just one sample of the return of any policy. But a set of $n$ full-horizon trajectories can be used to provide $n$ sample returns for a given policy $\pi$, thereby providing an empirical estimate of $V^{\pi}$. We can do this off policy evaluation of $V^{\pi}$ for any policy in the class $\Pi$.

Prior work has shown that given access to a generative model of the domain, policy search can be done in an efficient way by using common random numbers to evaluate policies that act differently in an episode \cite{DBLP:conf/nips/KearnsMN99,DBLP:conf/uai/NgJ00}. In our setting, a full-horizon trajectory is essentially equivalent to having access to a generative model that can produce a single return for any policy.
  However, access to a full length trajectory can be obtained by running in the environment, whereas generic generative models typically require "teleporation": the ability to simulate what would happen next under a particular action given an arbitrary prior history, which is hard unless in a planning scenario in which one already has knowedge of the dynamics process. Our results require weaker assumptions than prior results that use stronger generative models to obtain similar sample efficiency, while also achieving better sample efficiency than approaches with  access to similar generative models.

We will shortly provide a sufficient condition on the number 
of full length trajectories $n$, to guarantee that 
we can evaluate any policy sufficiently accurately 
to enable policy search to identify a near-optimal policy 
(within the input policy class). Of course, empirically, 
we will often wish to select a smaller $n$: our simulation 
experiments will demonstrate that often a small $n$ still 
enables us to identify a good policy. 

%A natural question is whether
%using the fixed set $\Gamma$ introduces bias in the evaluation of policies. As we show in experiments, for stochastic domains such as those described in this paper, the ability to use more trajectories for evaluation overrides such concerns. Using a distinct but small number of different trajectories for each policy (such as in \old) can lead to high estimation error and the selection of a bad policy.

%The next question that address is how to fix $n$ \emph{i.e.} the number of full trajectories to be collected. Our motivation in picking $n$ is straightforward; $n$ should be large enough so that policy value estimates $\hat{V}_\pi$ uniformly converge to their true values $V_\pi$ with high probability.

\section{Theoretical Analysis}
We now provide bounds on the sample complexity of 
\us: %specifically,
the number of full length trajectories
required to obtain near accurate estimates of all
policies in a policy class. This is sufficient to identify
the optimal (or near-optimal) policy in the policy class with the highest expected return.

First, we note that the optimal stopping problems we consider in this paper can be viewed as a particular instance of a POMDP. Briefly, there is some hidden state space, with a dynamics model that determines how the current state transitions to a new state stochastically, given the continue action. The observation is a function of the hidden state, and the reward is also a function of the hidden state and action. 

Our main result is that, given a policy class $\Pi$,
the sample complexity scales logarithmically with the
horizon. We make no assumption of access to a generative
model. This is a significant improvement over prior
sample complexity results for policy search for
generic POMDPs and large MDPs~\cite{DBLP:conf/nips/KearnsMN99,DBLP:conf/uai/NgJ00} which required access to a generative model of
the environment and had a sample complexity
that scaled linearly with the horizon.
These results can be thought of as bounding the
computation/simulation time required during \textit{planning},
when one has access to a generative model that can
be used to sample an outcome (reward, observation)
given any prior history and action. In contrast,
our results apply during \textit{learning}, where the agent
has no generative model of the domain, but must
instead explore to observe different outcomes. Without a generative model of the domain, sample complexity results for policy search in generic POMDPs when learning scale exponentially with the horizon \cite{DBLP:conf/nips/KearnsMN99}.

% We achieve these results due to the special structure
% of optimal stopping problems, where longer
% trajectories can be used to infer the outcomes
% of both halting where the trajectory did end, as
% well as all prior points.
Optimal stopping trajectories are %therefore
related to the trajectory trees of \citeauthor{DBLP:conf/nips/KearnsMN99}
which were used to evaluate the
returns of different POMDP policies.
For a POMDP with $2$ actions, each trajectory tree is a complete
binary tree (of depth $H$) rooted at a start state. Nodes in the
tree are labeled with a state and observation, and a path from
the root to any node in the tree denotes a series of actions taken
by a policy. A trajectory tree can be used
to evaluate any policy in $\Pi$, since every action sequence
%that a policy could take 
is part of the tree.
%\citeauthor{DBLP:conf/uai/NgJ00} discuss compressing this (exponential) trajectory tree representation using the method of \emph{common random numbers} -- each trajectory tree becomes a \emph{scenario} with a start state and $H$ random numbers appended to it.                             
%Constructing $\Gamma$ before policy search is similar to `fixing randomness' in the form of scenarios, as done in \citeauthor{DBLP:conf/uai/NgJ00}.                                                     
However, while for generic POMDPs the size of a trajectory tree
is exponential in the horizon, for optimal stopping problems the tree size
is linear in the horizon (Figure~\ref{fig:full_traj}).
This allows us to obtain significantly tighter dependence
on $H$ than for generic POMDPs.

Our analysis closely follows the 
prior sample complexity
results of \cite{DBLP:conf/nips/KearnsMN99}.
% for policy search
% in large MDPs and POMDPs
% using generative models~\cite{DBLP:conf/nips/KearnsMN99,DBLP:conf/uai/NgJ00}.
\citeauthor{DBLP:conf/nips/KearnsMN99} proceeded
by first considering a bound on the VC-dimension of
$\Pi$ when viewed as a set of real-valued mappings from
histories to returns, as a function of
 the VC-dimension of $\Pi$ when viewed as mappings of
histories to actions. Then they use this result to
bound the sample complexity needed to get near-accurate
estimates of the returns of all policies in the policy
class.

We will follow a similar procedure to bound
the sample complexity when $\Pi$ contains
a potentially infinite number of deterministic policies.\footnote{
Similar to Kearns et al.~(\citeyear{DBLP:conf/nips/KearnsMN99}) our results extend to
finite $\Pi$ and infinite, stochastic $\Pi$, as well as the discounted
infinite-horizon case (using an $\epsilon$-approximation to $V_\pi$
with horizon $H_\epsilon$).}
Let $d = \mathrm{VC}(\Pi)$ be the VC-dimension of
our policy class. This is well-defined, since each optimal-stopping
policy maps trajectories to 2 actions (binary labeling). Let
$\mathrm{VC}_r(\Pi)$ be the VC-dimension of $\Pi$ when viewed as a
set of real-valued mappings from full trajectories to returns and assume $V_\pi$ is bounded by $V_{\mathrm{max}}$. From
\cite{vapnik1982estimation}, we know that $\mathrm{VC}_r(\Pi)$
can be computed as $\mathrm{VC}(\mathcal{H})$,
 where $\mathcal{H} = \left\{\mathrm{I}(\pi,\Delta,\cdot) | \pi \in \Pi, \Delta \in [-V_{\max},V_{\max}] \right\}$
with $\mathrm{I}(\pi,\Delta,x) = 1$ if $\pi(x) \ge \Delta$, and
$0$ otherwise ($\pi(x)$ is the return for full trajectory $x$ under $\pi$).

\begin{newlemma}{1}
Let $\Pi$ be a set of deterministic optimal-stopping policies with
VC-dimension $d$ when viewed as a set of maps from trajectories to
actions. Then, when viewed as a set of maps from the space of all
full trajectories to $[-V_\mathrm{max},V_\mathrm{max}]$, $\Pi$ has
dimension bounded by,
\begin{equation*}
\mathrm{VC}_r (\Pi) = \mathrm{O}(d \log H)
\end{equation*}
\end{newlemma}

\begin{proof}
Our proof proceeds similarly to Lemma A.1 in
\citeauthor{DBLP:conf/nips/KearnsMN99}.
The crucial difference is
that our policies operate on a full-trajectory structure that
contains $H$ nodes (Figure~\ref{fig:full_traj}), rather
than \citeauthor{DBLP:conf/nips/KearnsMN99}'s trajectory trees with $2^{H+1}$ nodes. In our setting at each point the agent only gets to consider whether to halt or continue, and if the halt action is chosen, the trajectory terminates. This implies that in contrast to standard expectimax trees where the size of the tree depends on the action space as an exponential of the horizon, $|A|^{H}$, in our setting the dependence induced by the actions is only linear in $H$. Thus $\Pi$ can produce a much smaller set of behaviors, and our dependence on
$H$ is logarithmic, rather than polynomial.

More formally, by Sauer's lemma, $k$ trajectories can be labeled in atmost $(\frac{ek}{d})^d$ ways by $\Pi$. First note that $n$ \emph{full} trajectories contain at most $nH$ distinct trajectories across them (one per node; refer to Figure~\ref{fig:full_traj} for the structure of full trajectories). Each action labeling of these $k = nH$ trajectories by $\Pi$, corresponds to \emph{selecting} $n$ paths (1 path per full trajectory), where each path starts at the first observation and ends at a terminal node. The number of possible selections by $\Pi$ is thus atmost $(\frac{enH}{d})^d$. Each path can be viewed as mapping a full trajectory to a return, and a selection therefore maps the $n$ full trajectories to $n$ real-valued returns.

\begin{figure}[t]
    \centering
     \includegraphics[width=0.9\linewidth]{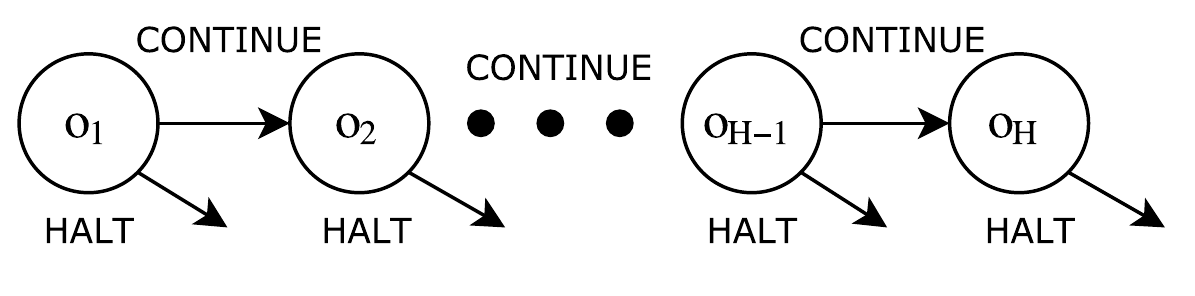}
    \caption{The structure of a full trajectory for horizon $H$. Each node represents an observation while arrows represent one of the two available actions.}
    \label{fig:full_traj}
%     \vspace{-6pt}
\end{figure}

There are $nH$ terminal nodes across the $n$ full trajectories. Thus there are at most $nH$ distinct real-valued returns on the $n$ full trajectories under $\Pi$. If we set the indicator threshold $\Delta$ to equal each of these $nH$ returns in turn, there would be atmost $(\frac{enH}{d})^d$ distinct binary labelings of the full trajectories for each such $\Delta$. Thus, the set of indicator functions that define $\mathrm{VC}_r(\Pi)$ can generate atmost $nH(\frac{enH}{d})^d$ distinct labelings on $n$ full trajectories. To shatter the $n$ full trajectories, we set $nH(\frac{enH}{d})^d \ge 2^n$, and the result follows.
\end{proof}

We now proceed similarly to Theorem 3.2 in \citeauthor{DBLP:conf/nips/KearnsMN99}

\begin{theorem}
\label{theorem:pac}
Let $\Pi$ be a potentially infinite set of deterministic optimal stopping policies and let $d$ be the VC-dimension of $\Pi$. Let $\Gamma$ be $n$ full trajectories collected from the environment, and let $\hat{V_\pi}$ be the value estimates for $\pi \in \Pi$ using $\Gamma$. Let the return $V_\pi$ be bounded by $V_\mathrm{max}$ for any trajectory. If
\begin{equation*}
n = \mathrm{O}\left( \left(\frac{V_{\mathrm{max}}}{\epsilon}\right)^2 \left(d \log H + \log \frac{1}{\delta} \right)  \right)
 \end{equation*}
then with probability at least $1 - \delta$, $|V_\pi - \hat{V}_\pi | \le \epsilon$ holds simultaneously for all $\pi \in \Pi$.
\end{theorem}

\begin{proof}
Let $X$ be the space of full trajectories. Every policy $\pi \in \Pi$ is a bounded real-valued map $\pi: X \to [-V_\mathrm{max},V_\mathrm{max}]$. Let $x_1,\dots,x_n \sim \mathrm{P}$ be i.i.d. full trajectories generated by the environment dynamics. Using a result of \cite{vapnik1982estimation}, we have with probability $1 - \delta$, $\sup_{\pi \in \Pi} \left|\mathrm{E_P}[\pi(x)] - \frac{1}{n}\sum_{i=1}^n \pi(x_i)\right| \le \mathrm{O}\left(V_\mathrm{max} \sqrt{ \frac{\mathrm{VC}_r(\Pi) \log \frac{n}{\mathrm{VC}_r(\Pi)} + \log \frac{1}{\delta}}{n} }\right)$. Substitute $\mathrm{E_P}[\pi(x)] = V_\pi$, $\frac{1}{n}\sum_{i=1}^n \pi(x_i) = \hat{V}_\pi$, $\mathrm{VC}_r(\Pi) = \mathrm{O}(d \log H)$ in the inequality and upper-bound by $\epsilon$ to get the result.% follows on upper-bounding by $\epsilon$.
\end{proof}
In practice it may be impossible for us to evaluate every policy in $\Pi$, and then select the one with the best estimated mean. In such cases, we can use a different search method ($\mathcal{S}$ in Algorithm~\ref{alg:our_approach}) to find a local optima in $\Pi$, while using our bound to ensure that policy values are estimated accurately.

Lastly, we discuss \cite{tsitsiklis1999optimal}, who estimate Q values for finite-horizon Markov optimal stopping problems using a linear combination of basis functions, and then use that to find a threshold policy. They outline a procedure to tune the basis function weights that asymptotically guarantees their policy value's convergence to the best basis-function approximation. Under their assumptions, if we construct a policy class using basis functions, we inherit the useful convergence results relying on their search procedure, along with retaining our finite sample complexity results.

\section{Experiments}
%The objective of our experiments is to both demonstrate that 
We now demonstrate the setting we consider is sufficiently general to capture several problems of interest and that our approach, \us, can improve performance in optimal stopping problems over some state-of-the-art baselines.

\subsubsection{Ticket Purchase}
Many purchasing problems can be posed as an optimal stopping process where the return from stopping is %a fully observable function of the observed history, 
simply the advertised cost. We consider deciding when to purchase an airline ticket for a later trip date in order to minimize cost. The opaque way in which prices are set, and competitive pricing makes this domain difficult to model. Prior work \cite{DBLP:conf/kdd/EtzioniTKY03,DBLP:journals/tist/GrovesG15} has focused on identifying features to create sophisticated models that make good purchase decisions. Surprisingly, it can be hard to improve on an earliest purchase baseline that buys after the first observation.% \cite{DBLP:journals/tist/GrovesG15}.

We use data from Groves and Gini (\citeyear{DBLP:journals/tist/GrovesG15}) who collected real pricing data for a fixed set of routes over a period of 2 years, querying travel sites regularly to collect price information. Each route has several departure dates distributed over the 2 year period. For a price observation sequence of length $T$, a customer could commence his ticket search at any point in the sequence (\emph{e.g.} some customer starts 60 days before departure while another only a week before). Thus, we consider all such commencement points separately to get $T$ distinct full trajectories (similar to \cite{DBLP:journals/tist/GrovesG15}).

We construct a parameterized policy class ($\Pi_{\mathrm{simple}}$) based on Ripper's decision rules in \cite{DBLP:conf/kdd/EtzioniTKY03}: {\sc wait} if ($\mathrm{curr\_price} > \theta_0$ {\sc and} $\mathrm{days\_to\_depart} > \theta_1$) else {\sc buy}, where {\sc buy} corresponds to halting. We also constructed a more complex class ($\Pi_{\mathrm{complex}}$) with 6 parameters, that learns different price thresholds depending on how far the departure date is. We consider nonstop flights on 3 routes, NYC-MSP, MSP-NYC and SEA-IAD, training/testing each separately.  

Our method, \us\ collects full length trajectories during the first 200 days ($\sim$$1000$ trajectories) and uses them to construct a single stopping policy.
%construct (different) optimal stopping policies for horizons $H=1,\ldots,60$. Note this is necessary because the horizon of the optimal stopping problem depends on the data of departure of an individual relative to the current "starting" data for searching for airline tickets. % departure dates occurring in the first 200 days as training data ($\sim$$1000$ trajectories) for policy search, and 
It performs a simple policy search by sampling and evaluating $500$ policies randomly from the policy space. It then uses the best identified policy to simulate ticket purchasing decisions  for departure dates occurring during the remaining part of the 2 years ($\sim$$2000$ trajectories). We restrict the data to departure dates that contain at-least 30 price observations.\footnote{We found that shorter trajectories were collected close to the departure date, where prices fluctuate more and for which our illustrative policy classes are inadequate. In such cases, our method adopted a risk-averse earliest purchase policy.} 
 %; earliest purchase buys immediately while latest purchase waits till the day of departure. %\cite{DBLP:journals/tist/GrovesG15} have indicated that earliest purchase can be a difficult baseline to better.
 
Results on the test sets are shown in Table~\ref{tab:airline}. Our policy search method succeeds in finding a policy that leads to non-trivial improvement over the difficult earliest purchase baseline. Our improvements are in line with prior approaches specifically designed for this particular domain.\footnote{Unfortunately, the authors were unable to provide us with the train/test split used in \cite{DBLP:journals/tist/GrovesG15}.}
\begin{table}[]
\centering
\caption{Mean expenditure of deploying different policies on the test set for ticket purchase. Earliest purchase buys immediately, latest purchase waits till the departure date.}
\label{tab:airline}
\begin{tabular}{@{}llll@{}}
 \toprule
Method            & \multicolumn{1}{c}{NYC-MSP} & MSP-NYC & SEA-IAD \\ \midrule
Ours ($\Pi_{\mathrm{simple}}$)     & \$355                       & \$374   & \$565   \\
Ours ($\Pi_{\mathrm{complex}}$)    & {\bf \$351}                       & {\bf \$344}   & {\bf \$560}   \\
Earliest purchase & \$380                       & \$383   & \$578   \\
Latest purchase   & \$631                       & \$647   & \$938   \\ \midrule
Best possible price           & \$307                       & \$306   & \$513   \\ \bottomrule
\end{tabular}
\end{table}
These results highlight how our setting can capture important 
purchasing tasks and how our approach, even with a simple policy search, can find policies with significantly better performance 
than competitive domain-specific baselines.

\subsubsection{Tutoring and Asset Replacement}
We now consider 2 simulated domains and compare \us\ to several approaches for learning 
to act quickly in these domains.  Unless specified, all results are averaged over 20 rounds and error bars indicate $95\%$ confidence intervals.

\noindent{\textbf{Baselines.}} One natural idea is to proceed as \us, but use the 
gathered data to build parametric domain models  
that can be used to estimate the performance of 
potential policies. We call these "model-based" approaches. 
A second idea is to consider the initial set of 
collected data as a budget of free exploration, 
and instead use this budget to do Monte Carlo on-policy evaluation 
of a set of policies. 

Of course, doing all exploration, as we do in \us, 
is not always optimal. We also consider a state-of-the-art 
approach for quickly identifying the global optima of a function where the function is initially unknown and each function evaluation is expensive, Bayesian Optimization (BO).
Multiple papers have shown BO can be used to speed online policy search for reinforcement learning tasks \cite{wilson2014using,deisenroth2011pilco}. %(**Add citations: Wilson, and prior). 
Given the policy class, BO selects a policy to evaluate at each step, and maintains estimates over the expected value of every policy. %is the expected any kind, and c return of the policy. 
We use Yelp's MOE for BO \cite{YelpMOE} with a 
%\cite{DBLP:conf/nips/SnoekLA12} using a 
Gaussian kernel and the popular expected improvement heuristic for picking policies from $\Pi$. The hyper-parameters for BO are picked by a separate optimization to find maximum-likelihood estimates. %We also consider a variant of BO, Model Based BO (MBOA) \cite{wilson2014using} which achieved state-of-the-art performance in a number of reinforcement learning domains: the key idea of MBOA is to leverage an input model class to improve the efficiency of online BO policy search. MBOA met or exceeded another popular policy-search method, PILCO \cite{deisenroth2011pilco}, but at significantly lower computational cost. 

\noindent{\textbf{Simulated Student Learning.}} We first consider a simulated student tutor domain. A number of tutoring systems use mastery teaching, in which a student is provided with practice examples until they are estimated to have mastered the material.  This is an optimal
stopping problem because at each time step, after observing
whether a student got the activity correct or not, the tutor 
can decide whether to halt or continue providing the student
with additional practice. On halting, the student is given the next problem in the sequence; the objective is to maximize the score on this `posttest', while giving as few problems as possible overall.
As is popular in the literature, we model student learning using the Bayesian Knowledge Tracing (BKT) model \cite{DBLP:journals/umuai/CorbettA95}. BKT is a 2-state Hidden Markov Model (HMM) with the state capturing whether the student has mastered the skill or not. Within the HMM, 4 probabilities -- $p_i$ (prior mastery), $p_t$ (transition to mastery), $p_g$ (guess) and $p_s$ (slip) describe the model. To simulate student data, we fix BKT parameters\footnote{Our results hold for other instantiations of these parameters as well. See \cite{DBLP:conf/edm/RitterHNDMT09} for other reasonable parameter settings.} $p_i = 0.18, p_t = 0.2, p_g = 0.2, p_s = 0.1$ and generate student trajectories using this BKT model for $H = 20$ problems. 

For \us, we consider two policy classes, both of which halt when the probability of the student's next response (according to the model in use) being correct crosses a threshold. Thus, we halt if $\Pr(o_t = \mathrm{correct}|o_1,\dots,o_{t-1}) > \theta_0$ where $\theta_0$ is some threshold. In fact, policies of this kind are widely used in commercial tutoring systems \cite{koedinger2013new}. %This procedure holds true more generally for optimal stopping problems, where a model-based predictor of the next state can be compared to a threshold parameter to create a stopping policy. 
If we use the BKT model to implement this policy class, it is parameterized by $(p_i,p_t,p_g,p_s,\theta_0)$. $\Pi$ then contains all possible instantiations of these parameters for our model-free approach to search over. We also consider a policy class based on another popular educational data mining model of student learning: Additive Factors Model (AFM) \cite{draney1995measurement}. AFM is a logistic regression model used to predict the probability that a student will get the next problem correct given their past responses. Thus, $\Pr(o_t = \mathrm{correct}|o_1,\dots,o_{t-1}) = \frac{1}{1 + e^{-(\beta_1 + \beta_2 n_c)}}$ where $n_c$ is the number of correct past attempts.

We first note that \us\ is significantly more effective than taking the same budget of exploration, and using it to evaluate each policy in an on-policy manner using Monte Carlo (MC) estimation. More precisely, we sample $k = 100$ policies from the BKT policy class, and %For both methods, we directly search over the BKT policy class in a model-free way. 
fix a budget of $B \in \{100,1000\}$ trajectories. %, which represents how many students we have.
\us\ uses $B$ trajectories to evaluates all $k$ policies while MC runs every policy on $\frac{B}{k}$ trajectories (\emph{e.g.} 1 trajectory/policy for $B = 100$) and selects the one with the highest mean performance. Averaging results across 20 separate runs, we found that \us\ identifies a much better policy; MC chose poor policies because it is mislead by the potential performance of a policy due to the limited data.

We also explored the performance of building a model of the domain, both in the setting when the model matches the true domain (here, a BKT model), and a model-mismatch case, where the policy class is based on a student AFM model (which does not match the BKT process dynamics). We use maximum likelihood estimation to fit the assumed model's parameters given the collected data 
%to find the best-fit $(p_i,p_t,p_g,p_s)$ (using Baum-Welch for the HMM), 
and then separately optimize over the threshold parameters $\theta_0$. %A similar approach works for the AFM model. 
We compare, on varying the budget $B$: (a) \us; (b) model-based; (c) BO. All results are averaged over 50 trials.

%(a) Our model-free approach, using search over 500 random policies in $\Pi$. (b) The model-based approach described earlier. (c) Bayesian Optimization (BO) using a Gaussian kernel and the Expected Improvement heuristic for picking policies from $\Pi$. The hyper-parameters for BO were picked by a separate optimization routine that uses maximum-likelihood estimates. The BO implementation used was Yelp's MOE \cite{YelpMOE}.

\begin{figure}[t]
    \centering
    \begin{subfigure}{0.5\linewidth}
        \centering
        \includegraphics[width=\textwidth]{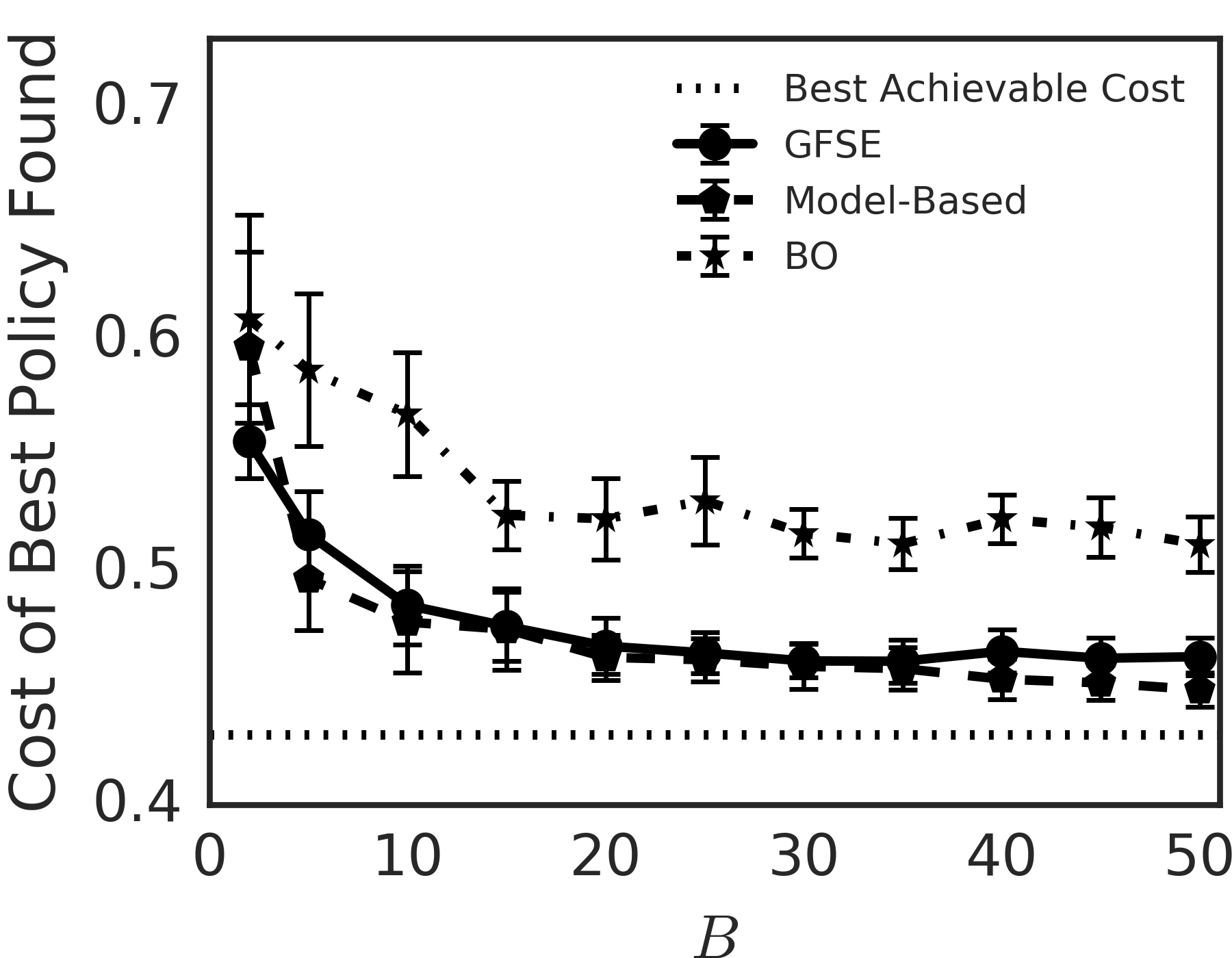}
%         \caption{}
    \end{subfigure}%
    \begin{subfigure}{0.5\linewidth}
        \centering
        \includegraphics[width=\textwidth]{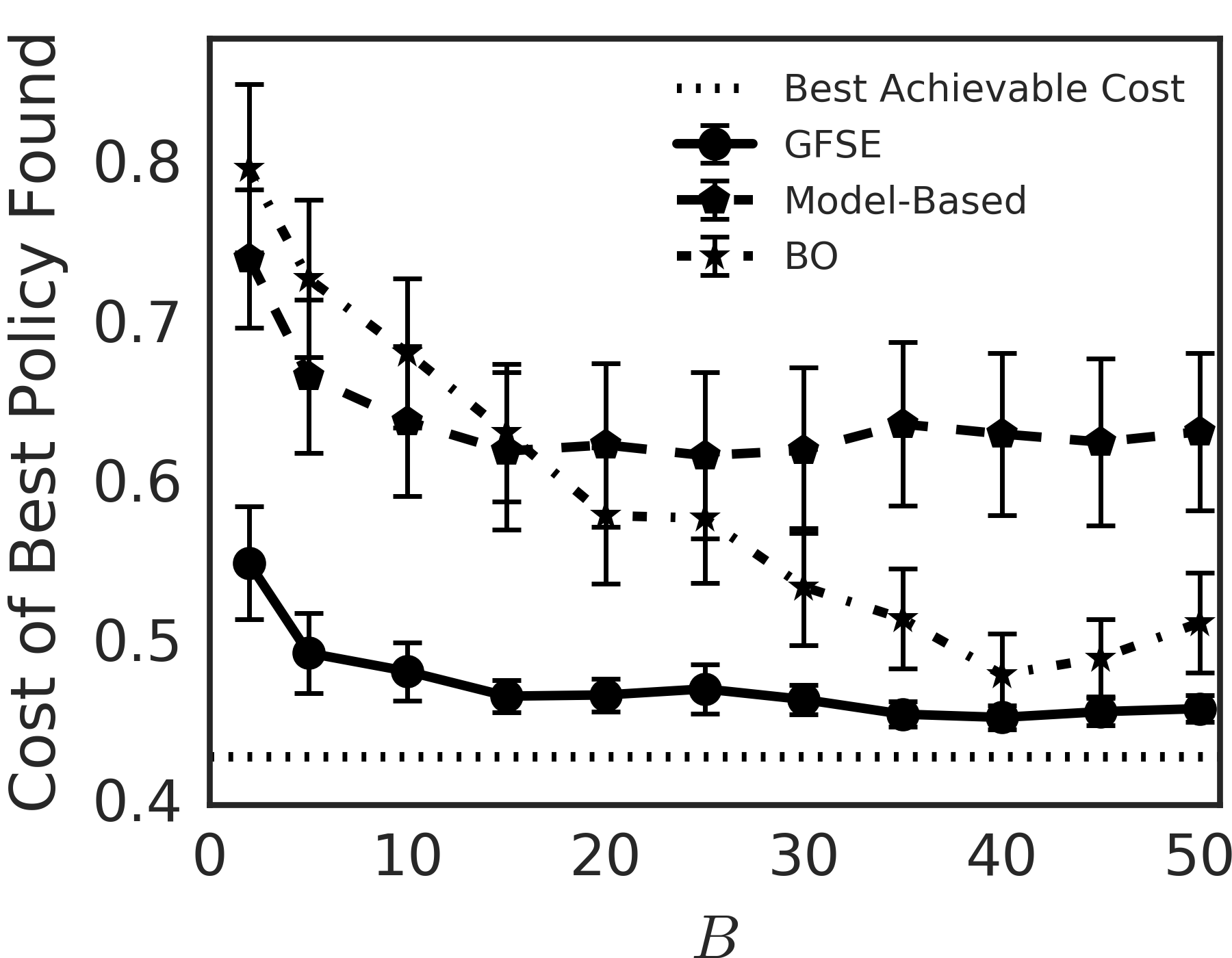}
%         \caption{}
    \end{subfigure}
    \caption{Comparison of best policy found on varying budget under (a) matched model setting; (b) model mismatch setting.}
    \label{fig:tutoring_comparison}
\end{figure}

The results of this experiment are shown in Figure~\ref{fig:tutoring_comparison}. Our approach does well in both settings, quickly finding a near optimal policy. As one would expect, the model-based approach does well under the matched model setting, making full use of the knowledge of the underlying process dynamics. However, on fitting the mismatched AFM model, the model-based approach suffers. As has been noted by prior work \cite{mandel2014offline} 
model-fitting procedures focus on maximizing the likelihood of the observed data rather than trying to directly identify a policy that is expected to perform well. BO can find a good policy, but takes more samples to do so. 

%This is somewhat tangential to our final goal, which is simply to find a policy with high value. For hard to model domains, carefully fitting a mismatched model's parameters is unlikely to yield a good policy. In such cases we may not need to find the best-fitting model, in order to find a high-value policy. This is important for intelligent tutoring, where it is prevalent to fit a predictive model to the data, before converting it into an instructional policy by combining it with an optimized threshold.

%We note that for both settings, BO is also able to find a policy with low cost, but at a higher budget compared to our method. BO does not reuse trajectories like our method, but does share information through the Gaussian Process it maintains. Our method does well in both settings, 

\begin{figure}[t]
    \centering
    \begin{subfigure}{0.5\linewidth}
        \centering
        \includegraphics[width=\textwidth]{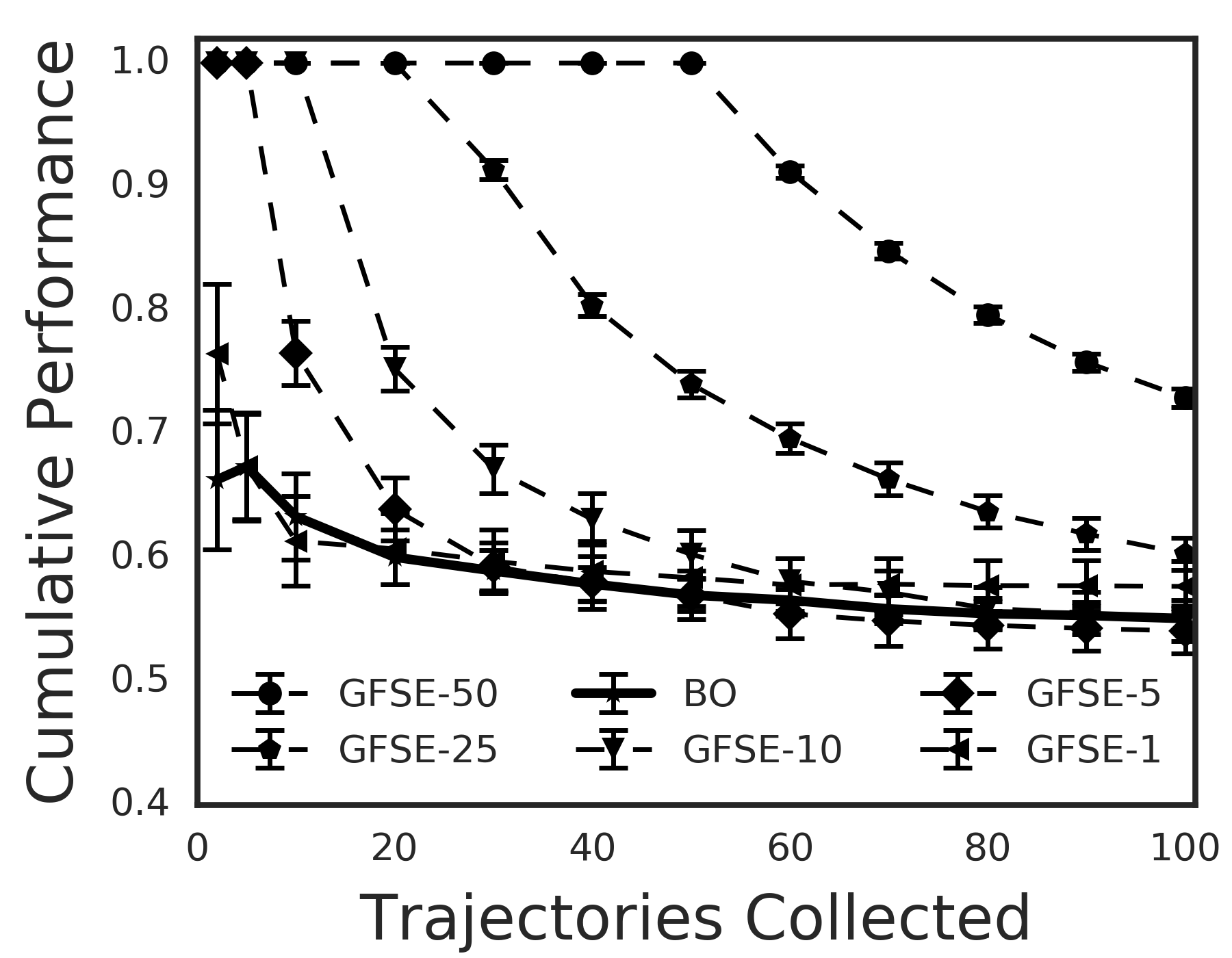}
%         \caption{}
    \end{subfigure}%
    \begin{subfigure}{0.5\linewidth}
        \centering
        \includegraphics[width=\textwidth]{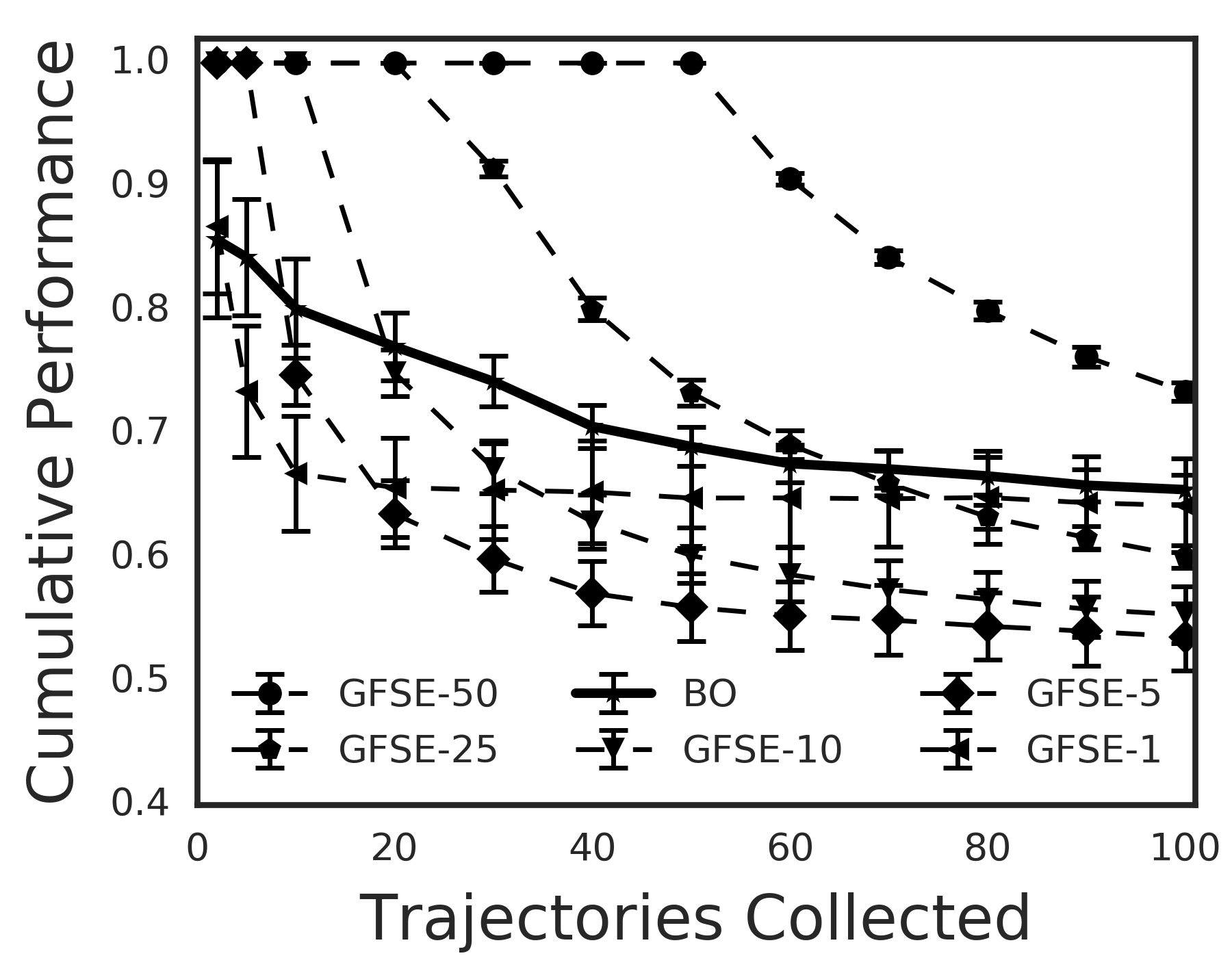}
%         \caption{}
    \end{subfigure}
    \caption{Average cumulative performance on the simulated student domain. (a) matched model; (b) model-mismatch. GFSE-$k$ collects $k$ initial full-length trajectories before identifying and then executing the best policy. }
    \label{fig:cume_tutoring_comparison}
\end{figure}

Since BO is an online approach whereas \us\ uses a fixed budget of exploration, we also compare the averaged cumulative performance of BO to variants of \us\ in Figure~\ref{fig:cume_tutoring_comparison}. This mimics a scenario where we care about online performance on every individual trajectory, rather than having access to a fixed budget before deploying a policy. For our method, we can choose to collect more or less full trajectories before finding the best policy. Interestingly, if we use 5 trajectories as the initial budget to collect full length trajectories, \us\ meets or exceeds BO performance in this setting in both the matched and mismatched model cases, within 10 trajectories. 
%**Add Wilson discussion. 

BO suffers from the highly stochastic returns of policies in this setting. For more efficient data reuse, we also consider a variant of BO (BO-REuse) where we evaluate each proposed policy online and also using previously collected trajectories, yielding a more robust estimate of the policy's performance. Similarly, for \us\ we deploy a policy using an initial budget of trajectories, and then use its on-policy trajectory (in addition to earlier trajectories) to rerun policy search and identify another policy for the next time step (\us-RE). Figure~\ref{fig:augment_bo} shows this improved both methods (especially in the mismatch case), with our approaches still performing best. %MBOA does comparable to BO, possibly because it fits and then leverages a mismatched model to generate data.

\begin{figure}[t]
    \centering
    \begin{subfigure}{0.5\linewidth}
        \centering
        \includegraphics[width=\textwidth]{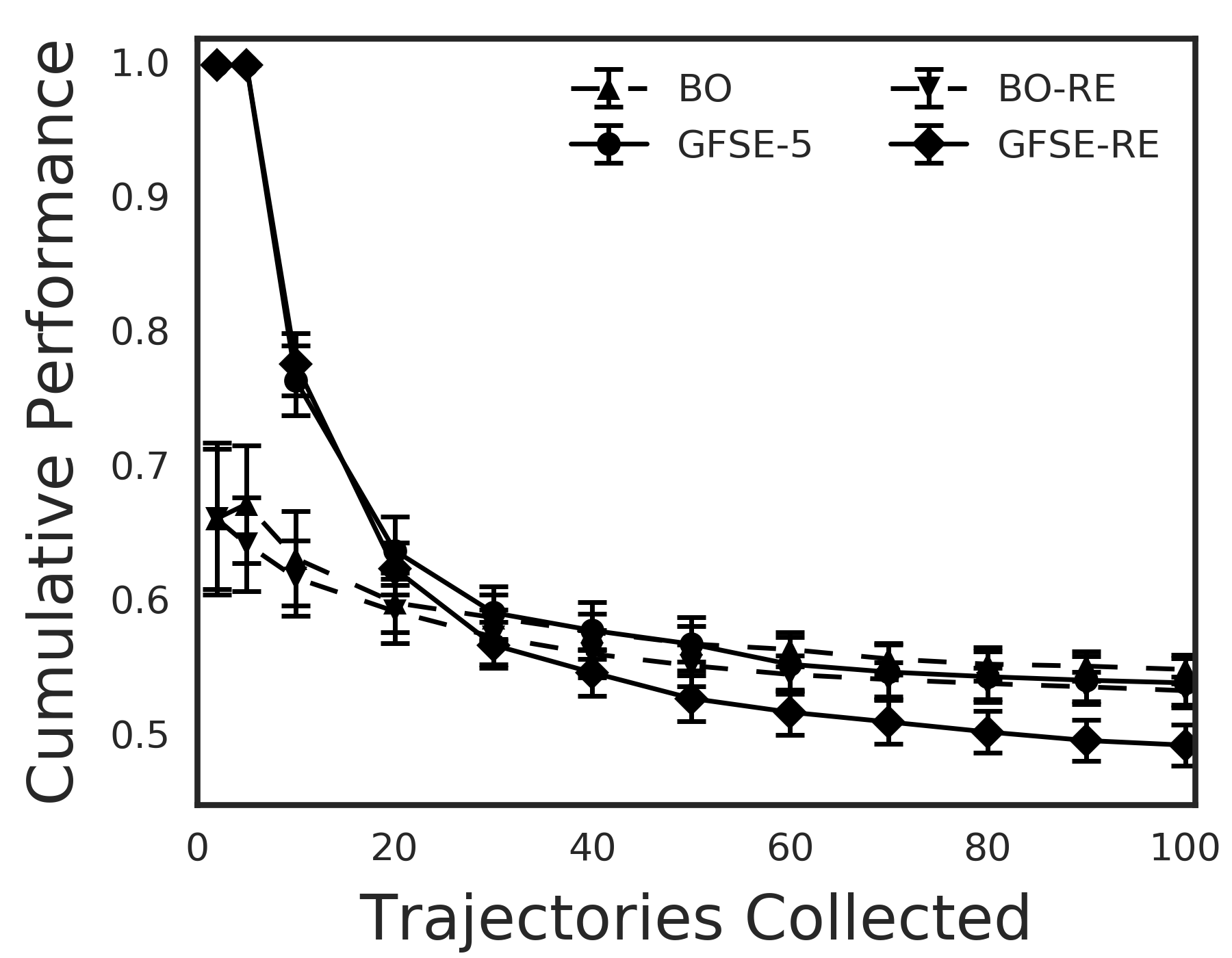}
%         \caption{}
    \end{subfigure}%
    \begin{subfigure}{0.5\linewidth}
        \centering
        \includegraphics[width=\textwidth]{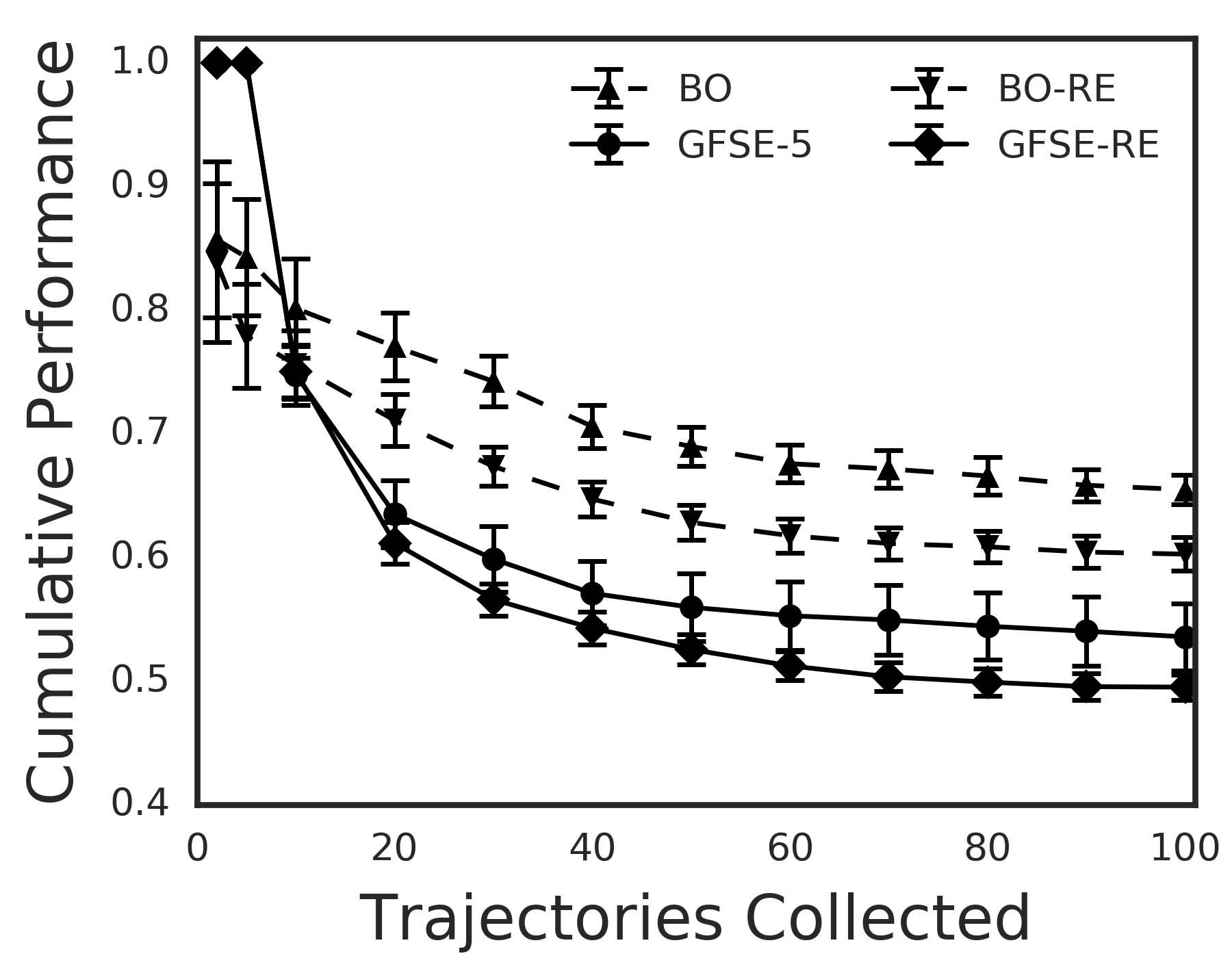}
%         \caption{}
    \end{subfigure}
    \caption{Cumulative performance of augmented methods on (a) matched model; (b) model-mismatch.}
    \label{fig:augment_bo}
\end{figure}

\noindent{\textbf{Asset Replacement}}. Another natural problem that falls into our optimal stopping problem is when to replace a depreciating asset (such as a car, machine, etc). For simulation, we use a model described in \cite{DBLP:journals/ior/JiangP15}. Variants of this model are widely used in that field \cite{feldstein1974towards,rust1987optimal}. In the model, observations are $d$ dimensional vectors of the form $(X,Y_1,\dots,Y_{d-1})$. Each asset starts at a fixed valuation $X = X_{\mathrm{max}}$ which depreciates stochastically\footnote{Details of the model can be found in \cite{DBLP:journals/ior/JiangP15}.} while emitting observations $Y_1,\dots,Y_{d-1}$ at every time step. The reward function used incorporates the cost of replacement (which increases over time), the utility derived from the asset and a penalty if the asset becomes worthless before replacement. We use $d=3$ for experiments.

We construct a logistic threshold policy class; replacing the asset if $\frac{1}{1 + e^{-(\beta_1 + \beta_2 \cdot \mathrm{depr})}} > \beta_3$ where $\mathrm{depr}$ is the total depreciation from $X_{\max}$ seen so far (normalized to lie in $[0,1]$). In addition to the approaches seen before, we also include baseline policies that choose to (i) replace the asset immediately; (ii) never replace. Lastly, we include the optimal value (known only in hindsight) for reference.

The results are shown in Fig~\ref{fig:opt_repl}. Surprisingly, our method outperforms competing methods by a considerable margin. It appears that our chosen policy class is tricky to optimize over: most policies in the space perform poorly. For 500 random policies chosen from this space, the mean cost is around 240 with a $95\%$ confidence interval of only 16. However, the domain itself is not very noisy, with robust value estimation requiring less than 5 trajectories (see Figure~\ref{fig:opt_repl}). This enables our method to consistently find a good policy even with a low budget: one that corresponds to replacing the asset when depreciation is around $50\%$. BO %(and even MBOA with the true model) 
improves slowly; either sampling bad policies due to the sparse nature of the space, or disbelieving the estimate of a good policy due to the bad policies surrounding it. Manually adjusting the BO  hyperparameters to account for this did not improve performance significantly. % (information is shared through the GP in BO). %We also used a recent BO model from \cite{wilson2014using} but found little improvement.

\begin{figure}[t]
    \centering
    \begin{subfigure}{0.5\linewidth}
        \centering
        \includegraphics[width=\textwidth]{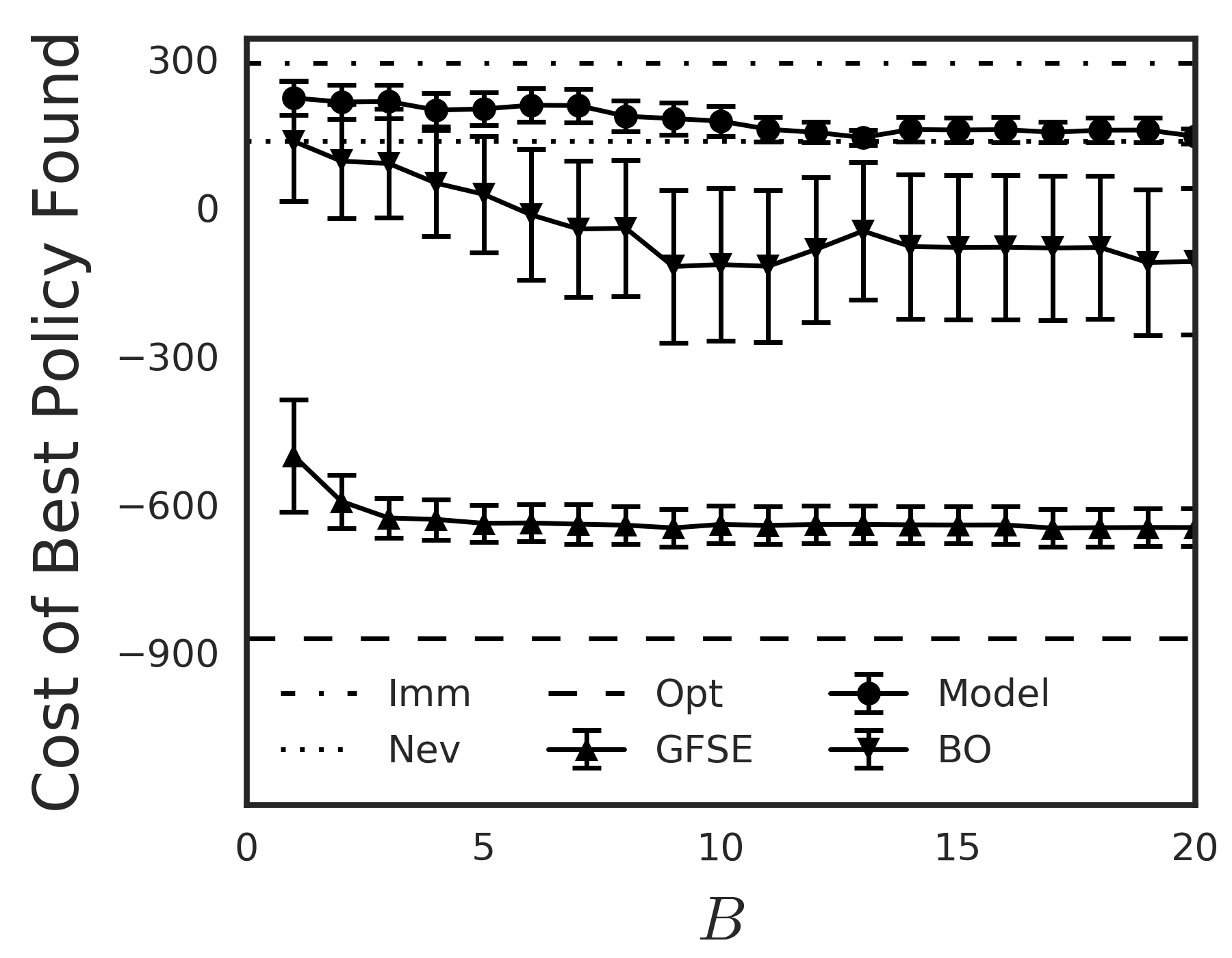}
%          \caption{}
    \end{subfigure}%
    \begin{subfigure}{0.5\linewidth}
      \centering
        \includegraphics[width=\textwidth]{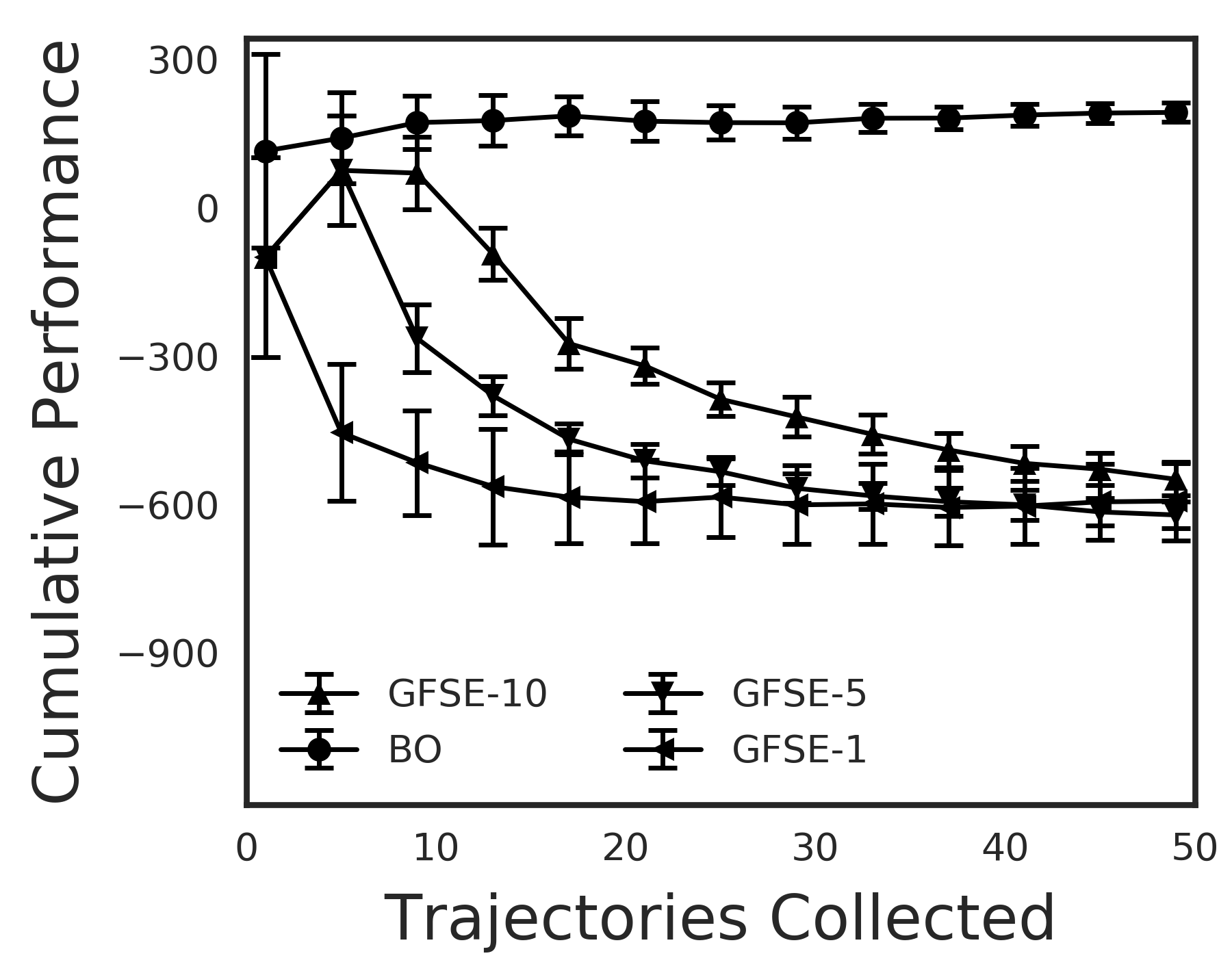}
%          \caption{}
    \end{subfigure}
    \caption{Results from asset replacement.}
    \label{fig:opt_repl}
\end{figure}

 \section{Discussion and Conclusion}
\us\ performed well, outperforming state-of-the-art algorithms and common baselines, in a variety of simulations of important domains. %We did not compare to recent policy search methods such as \cite{levine2014learning}, since they are more suited for robotics applications with smooth, kinematic trajectories. 
%Our experiments explored the benefit of doing direct policy search over a policy class in hard-to-model optimal stopping domains. 
%WImportantly, w
While we randomly searched over policies in relatively simple policy classes for illustration, more sophisticated search methods and policy classes could be employed, without effecting the theoretical guarantees we derived. Another extension is in using shorter trajectories that terminate before the horizon for policy evaluation (similar to how full trajectories are used). This is useful in a scenario where we get trajectories on-policy using the best policy found by \us. We can then rerun policy search with all trajectories (full length or short) collected so far. %Note that o
Our policy value estimates will be biased in this case, since only a policy that halts earlier than a shorter trajectory can use it for evaluation. Values for policies that halt later may be overestimated (higher variance of estimation due to fewer trajectories), biasing us to pick them. If the number of evaluations per policy exceeds the number in Theorem~\ref{theorem:pac}, our estimates would remain within $\epsilon$ of the true values (with high probability), which would minimize the effect of this bias. As we saw from Figure~\ref{fig:augment_bo}, this (\us-RE) works well empirically.

To summarize, we introduced a method for learning to act in optimal stopping problems, which reuses full length trajectories to perform policy search. Our theoretical analysis and empirical simulations demonstrate that this simple observation can lead to benefits in sample complexity and practice.% for a variety of optimal stopping problems. 

% It is also possible for us to use trajectories that are not full length for evaluation -- however, in this case we only this can be useful when running \us\ in a cumulative reward setting

\section{Acknowledgments}
We appreciate the financial support of a NSF BigData award \#1546510, a Google research award and a Yahoo gift. 

%\fontsize{9.0pt}{10.0pt}
%\selectfont
\bibliographystyle{named} 
\bibliography{bibfile}

\end{document}